\documentclass[11pt,a4paper]{article}
\usepackage[hyperref]{acl2021}
\usepackage{times}
\usepackage{latexsym}
\usepackage{prftree}
\usepackage{amsmath}
\usepackage{amsfonts}
\usepackage{amsthm}
\usepackage{amssymb}
\usepackage{graphicx}
\usepackage{stmaryrd}
\usepackage{booktabs}
\usepackage{times}
\usepackage{latexsym}
\usepackage{color}
\usepackage{float}
\usepackage{tikz}
\usepackage{tikz-cd}
\usepackage{latexsym}
\usepackage{xcolor}
\usepackage{enumitem}
\usepackage{lingmacros}
\usepackage{multirow}
\usepackage{booktabs}
\usepackage{subfig}

\newtheorem{thm}{Theorem}
\newtheorem{lemm}[thm]{Lemma}
\theoremstyle{definition}
\newtheorem{mydef}{Definition}[section]
\newcommand{\interp}[1]{\llbracket #1 \rrbracket}

\usetikzlibrary{shapes,decorations,arrows,calc,arrows.meta,fit,positioning}
\tikzset{
    -Latex,auto,node distance =1 cm and 1 cm,semithick,
    state/.style ={ellipse, draw, minimum width = 0.7 cm},
    point/.style = {circle, draw, inner sep=0.04cm,fill,node contents={}},
    bidirected/.style={Latex-Latex,dashed},
    el/.style = {inner sep=2pt, align=left, sloped}
}
\usepackage{microtype}

\aclfinalcopy 

\title{Supporting Context Monotonicity Abstractions in Neural NLI Models}

\author{Julia Rozanova$~^{\dagger}$, Deborah Ferreira$^{\dagger}$, Mokanarangan Thayaparan$^{\dagger}$,\\ \textbf{Marco Valentino$^{\dagger}$, Andr\'e Freitas$^{\dagger}$$^{\ddagger}$} \\  Department of Computer Science, University of Manchester, United Kingdom$^{\dagger}$ \\  Idiap Research Institute, Switzerland$^{\ddagger}$ \\ {\tt \{firstname.surname\}} {\tt @manchester.ac.uk} \\}

\begin{document}
\maketitle
\begin{abstract}
Natural language \textbf{contexts} display logical regularities with respect to substitutions of 
related concepts: these are captured in a functional order-theoretic property called \emph{monotonicity}.
For a certain class of NLI problems where the resulting entailment label depends only on the context monotonicity
and the relation between the substituted concepts, we build on previous techniques that aim to improve the
performance of NLI models for these problems, 
as consistent performance across both upward and downward monotone contexts
still seems difficult to attain even for state of the art models.  
To this end, we reframe the problem of \textbf{context monotonicity classification} to make it 
compatible with transformer-based pre-trained NLI models and add this task to the training pipeline. 
Furthermore, we introduce a sound and complete simplified monotonicity logic formalism which 
describes our treatment of contexts as abstract units. 
Using the notions in our formalism, we adapt targeted challenge sets to investigate 
whether an intermediate context monotonicity classification task can aid NLI models' performance
on examples exhibiting monotonicity reasoning.  
\end{abstract}
\section{Introduction}

NLI has seen much success in terms of performance on large benchmark datasets, but there are still
expected systematic reasoning patterns that we fail to observe in the state of the art NLI models.
We focus in particular on \emph{monotonicity reasoning}: a large class of NLI problems that can be
described as a form of \emph{substitutional} reasoning which displays logical regularities with
respect to substitution of related concepts. 
In this setting, a subphrase $\mathbf{a}$ of a premise 
$p(\mathbf{a})$ is replaced with a phrase $\mathbf{b}$, yielding a hypothesis
$p(\mathbf{b})$. 

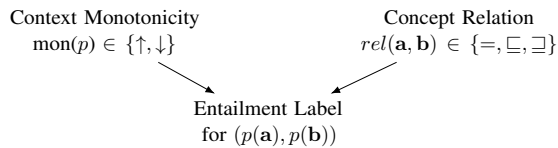
\begin{figure}[h!]
    \centering
\resizebox{\columnwidth}{!}{%
    \begin{tikzpicture}
            \node (1) [text width=3.5cm, align=center] at (0,0) {Context Monotonicity \\ mon($p$) $\in \{ \uparrow, \downarrow \}$};
        \node (2) [right = of 1]{};
        \node (3) [text width=4.5cm, align=center, right = of 2]{Concept Relation \\ $ rel(\mathbf{a}, \mathbf{b})\in \{
        =, \sqsubseteq, \sqsupseteq \} $};
        \node (4) [text width=3.5cm, align=center, below = of 2] {Entailment Label for $(p(\mathbf{a}), p(\mathbf{b}))$};
        \path (3) edge (4);
        \path (1) edge (4);
    \end{tikzpicture}%
    }
    \caption{The class of entailment problems under consideration: premise-hypothesis pairs 
   $(p(\mathbf{a}), p(\mathbf{b}))$
    whose entailment label depends only on the monotonicity of the context $p$ and the relation between $\mathbf{a}$ and $\mathbf{b}$.}
\end{figure}

Usually, the resulting entailment label relies on exactly two properties: the inclusion relation between concepts $\mathbf{a}$ and $\mathbf{b}$, and the \emph{systematic behaviour} 
of the context $p$ with respect to such relations.

In formal semantics, this is 
referred to as 
the \emph{monotonicity} of $p$ (where $p$ is either upward or downward monotone), 
and this reasoning pattern is referred to as \emph{monotonicity reasoning}.
Monotonicity reasoning is incredibly systematic, and thus is a much-probed behaviour in enquiries 
into the \emph{systematicity} and \emph{generalization capability} of neural NLI 
models \cite{goodwin-etal-2020-probing, yanakaMED, yanakaHELP, larryFragments, geiger}. 

Determining both the concept relation and the context monotonicity requires significant
linguistic understanding of syntactic structure and scope of operators, but in 
terms of reasoning, this is a very systematic pattern that nevertheless has a history of
causing problems for neural models.
It has been observed \cite{yanakaHELP, geiger} that current state of the art transformer-based NLI
models tend to routinely fail in downward
monotone contexts, such as those arising in the presence of negation or generalized quantifiers.
Recent strategies \cite{larryFragments} to address the shortcomings of NLI models in downward-monotone contexts have followed
the
\emph{inoculation} method \cite{inoculation}: additional training data which exhibits the target 
phenomenon (in this case, downward-monotone reasoning) is used to fine-tune existing models. 
This is done with some success in \cite{yanakaHELP, larryFragments} and \cite{geiger}. 
In contrast, we wish to investigate a \emph{transfer learning} strategy that directly targets the monotonicity question 
as an \emph{additional training task} to see if this can \emph{further} improve the monotonicity reasoning
performance of 
popular transformer-based NLI models.

Our contributions are as follows:

\begin{itemize}

    \item{Emphasizing monotonicity as a property of a \emph{context}, we introduce a sound and complete logical formalism which models the monotonicity reasoning phenomenon but abstracts away
    from specific linguistic operators, treating the context as a single abstract object.}
    
    \item Extending our treatment of contexts as individual objects to an experimental setting,
    we introduce an improvement in neural NLI model performance on monotonicity reasoning 
    challenge datasets by employing a
    context monotonicity classification task in the training pipieline of NLI models.
    To the best of our knowledge, this is the first use of neural models for this specific task.
    
    \item For this purpose, we adapt the \textbf{HELP} dataset \cite{yanakaHELP} into 
    a \textbf{HELP-Contexts} dataset, mimicking the structure of our logical formalism. 
    
    \item For the class of NLI problems described as \emph{monotonicity reasoning}, we demonstrate the impact of the proposed transfer strategy:
    we show that there can be a strong improvement 
    on downward monotone contexts, previously known to be a bottleneck for neural NLI models.
    As such, this shows the benefit of directly targeting intermediate abstractions 
    (in this case, monotonicity) present in logical formalisms that underly the true label. 
    
\end{itemize}
\section{Contexts and Monotonicity}
\subsection{Contexts}
Informally, we treat a natural language \emph{context} as a sentence with a ``gap", 
represented by a variable symbol. 

\noindent\textbf{A context $p(x)$}:

I ate some $x$ for breakfast. 

\noindent \textbf{A sentence} $S = p( \mbox{`fruit'})$:
    
I ate some fruit for breakfast. \\

Although every sentence can be viewed as a context with an insertion in as many 
ways as there are n-grams in the sentence, in this work we shall consider in 
particular those contexts where the variable corresponds to a slot in the expression
that may be filled by an \emph{entity} reference, such as a noun or noun phrase.
In the view of Montague-style formal semantics, these contexts correspond to expressions of type $<e,t>$.
    
\subsection{Monotonicity}
Given a natural language context $p$ and a pair of nouns/noun phrases $(\textbf{a}, \textbf{b})$, 
we may create a natural language sentence pair $(p(\textbf{a}),p(\textbf{b}))$ by substituting 
the respective subphrases into the natural language context.
When treated as a premise-hypothesis pair (as in the experimental NLI task setting), 
the gold entailment label has 
a strong relationship with the kinds of \emph{relations} that exist between the insertions $\mathbf{a}$ and $\mathbf{b}$.

In the seminal works on monotonicity \cite{valencia, vanbenthem}, the relations that are studied are \emph{semantic containment}
relations, which are defined analogously to set-theoretic containment relations ($\subseteq$).

\begin{table}[h]
\resizebox{\columnwidth}{!}{
\begin{tabular}{@{}lll@{}}
\toprule
                             & $\mathbf{a}$                            & $\mathbf{b}$              \\ \midrule
$\equiv$                      & couch                        & sofa           \\ \midrule
\multirow{3}{*}{$\sqsubset$} & apples                       & fruit          \\
                             & South African soccer players & soccer players \\
                             & dogs with hats               & dogs           \\ \bottomrule
\end{tabular}%
}
\caption{Examples of the semantic containment relation between concept pairs.}
\end{table}

For insertions related by $\sqsubset$, the gold entailment label depends on one other property:
the combined \emph{monotonicity profile} of all the linguistic operators within whose scope
the insertion is located. If the final monotonicity marking in the insertion's position is ``upward",
the gold label is entailment. However, if it is ``downward", we can deduce entailment of the 
reversed sentence pair, $(p (\mathbf{b}), p(\mathbf{a}))$. 
Linguistic operators such as "not" are downward monotone, while generalized
quantifiers such as ``every" have a more complex monotonicity profile: downward-monotone in the first argument
and upward-monotone in the second argument.
The monotonicity properties of all the operators compose along the syntax tree, culminating in a final
monotonicity marking for the ``$x$" position in the context (the monotonicity is independent of the
inserted word). 
\begin{figure}[h!]
    \centering
\resizebox{0.8\columnwidth}{!}{%
    \includegraphics{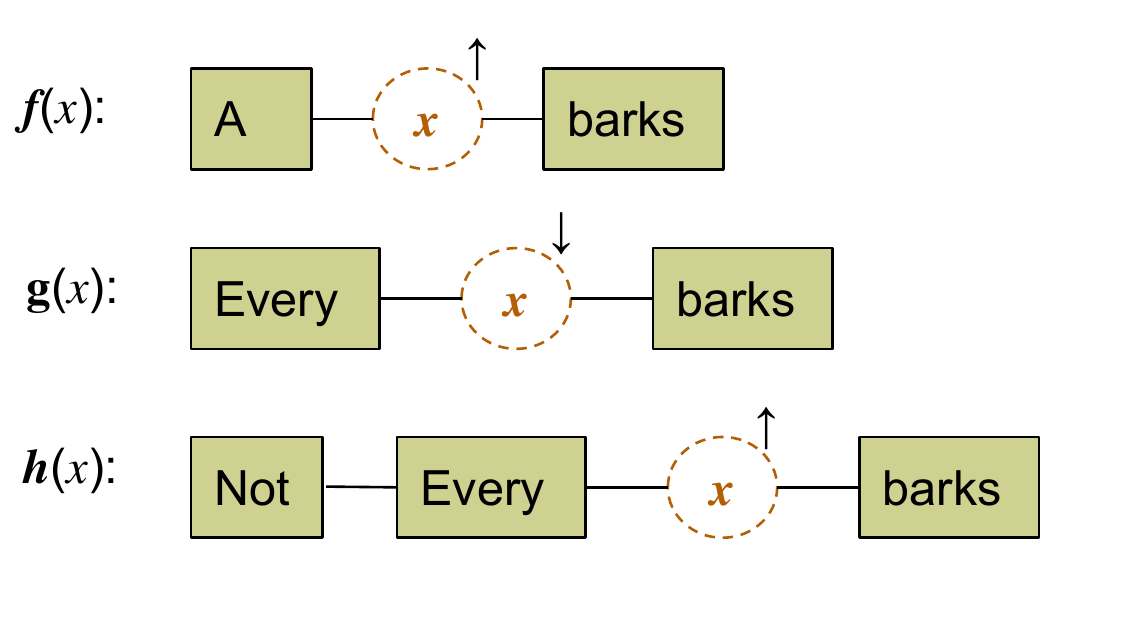}%
}
    \caption{Natural language \emph{contexts} have a property which dictates logical regularities with respect to concept hierarchies: like numerical functions, they can be upward monotone or downward monotone.  
    }
\end{figure}
It is this final monotonicity label that determines the entailment patterns with respect to insertion 
relations. Although there are formalisms that model this logical behaviour \cite{icard-etal-2017-monotonicity}, they aim to model the behaviour of 
each linguistic linguistic operator and the way they 
compose given the parse tree of a sentence. 

We consider a simplification of this behaviour by abstracting away the linguistic specifics of the 
context, treating it as a single abstract object. As such, we are not concerned with the exact monotonicity profiles of 
all the linguistic operators that culminate in the monotonicity of the final context.
We describe this behaviour with a simple logic system extending the $\mathcal{L}(all)$ 
logic of \cite{mossall} with the abstracted behaviour of upward and donward monotone contexts.
We include a proof that this adaptation is still sound and complete. 

\subsection{A Context-Abstracted Monotonicity Logic}
We extend the syllogistic syntax of the language $\mathcal{L}(all)$ included in \cite{syllogistic} and \cite{mossall}. 
In keeping with that style, we present the syntax as natural language sentences. However, we include
the corresponding first order formulae as well. In the subsequent proofs, we mix the stylizations 
somewhat for readability, but the table below should serve as a reminder for the exact correspondence. 
\begin{mydef}
Let the language $\mathcal{L}$ consist of the following:
\begin{enumerate}
    \item A countable set $\mathbf{A}$ of constant symbols $\textbf{a}, \textbf{b}, \textbf{a}_1, \textbf{b}_1, \ldots$
    \item Exactly two variables, $x$ and $y$
    \item A binary relation symbol $\sqsubseteq$.
    \item A set $\mathbf{P}$ of relation symbols $\sqsubseteq_p$ indexed by a countable set ~$p, p_1, \ldots$.
\end{enumerate}

\textbf{Only} the following are considered sentences in the language $\mathcal{L}$:
\begin{table}[h]
    \centering
    \resizebox{\columnwidth}{!}{
    \begin{tabular}{p{0.5\columnwidth}|l}
        Natural Language Stylization & FOL Stylization \\
        \hline
        all $\mathbf{a}$ are $\mathbf{b}$ & $\mathbf{a} \sqsubseteq \mathbf{b}$\\
        if $p(\mathbf{a})$ then $p(\mathbf{b})^*$ & $\mathbf{a} \sqsubseteq_p \mathbf{b}^*$\\
        $p$ is upward monotone & $\forall_{x,y} (x \sqsubseteq y \leftrightarrow x \sqsubseteq_p y)$\\
        $p$ is downward monotone & $\forall_{x,y} (x \sqsubseteq y \leftrightarrow y \sqsubseteq_p x)$\\
    \end{tabular}%
    }
    \caption{
    $^*$ For every natural language context $p$ in a set $P$ of contexts, and where $p(\mathbf{a})$ is the substitution of $\mathbf{a}$ into the natural language context $p$.}
    \label{tab:my_label}
\end{table}
\end{mydef}

This can in many ways be seen as a simplification of previous formalisms \cite{icard-etal-2017-monotonicity, hu-moss} based on either extending 
the syllogistic logic $\mathcal{L}(all)$ \cite{mossall} or extending typed lambda calculus with monotonicity behaviour. 
The key difference of this approach is the abstraction away from specific linguistic operators and their monotonicity profiles. 
On one hand, we are thus only modeling one level of linguistic
compositionality, but since the monotonicity profile of every linguistic operator composes into one monotonicity marker which affects 
the final entailment label (for this class of problem), it encompasses all of the linguistically-specific approaches.
This is useful when the monotonicity of a context can be determined by an external system such as a neural classifier or the ccg-2-mono system \cite{hu-moss} . 
In this case, the monotonicity marking of the entire context is explicit. 

\subsection{Semantics}
\begin{mydef}
A model $\mathcal{M}$ of the language $\mathcal{L}$ is
a structure
$$ \mathcal{M} = (M, \llbracket \cdot \rrbracket )$$
consisting of a set $M$ and an interpretation function  $\llbracket \cdot \rrbracket $
where $\llbracket \mathbf{a} \rrbracket \subseteq M$, $ \llbracket \sqsubseteq \rrbracket $
is the $\subseteq$ relation on the powerset $\mathcal{P}(M)$
and $\llbracket \sqsubseteq_p \rrbracket \subseteq \mathcal{P}(M)\times 
\mathcal{P}(M)$ 
is any binary relation on $\mathcal{P}(M)$.
Truth of a formula with respect to a given model is defined as follows: 

\begin{equation}
\small
    \mathcal{M} \models  \mathbf{a} \sqsubseteq \mathbf{b} \iff \interp{\mathbf{a}}
    \subseteq \interp{\mathbf{b} }
\end{equation}
\begin{equation}
\small
    \mathcal{M} \models  \mathbf{a} \sqsubseteq_p \mathbf{b} \iff \interp{\mathbf{a}},\interp{\mathbf{b}} \in \interp{\sqsubseteq_p}
\end{equation}
\begin{equation}
\small
    \mathcal{M} \models \forall_{x,y} 
    (x \sqsubseteq y \leftrightarrow x \sqsubseteq_p y) \iff \subseteq = \interp{\sqsubseteq_p}
\end{equation}
\begin{equation}
\small
    \mathcal{M} \models \forall_{x,y} 
    (x \sqsubseteq y \leftrightarrow y \sqsubseteq_p x)
    \iff \supseteq = \interp{\sqsubseteq_p}
\end{equation}

\end{mydef}

\subsection{Proof Calculus}
Our language will be equipped with the following deduction rules and axioms:

\begin{displaymath}
\prftree[r]{$\scriptstyle BARBARA$}
{\textsc{all } ~\mathbf{a} ~\textsc{are } \mathbf{b} }
{\textsc{all } ~\mathbf{b} ~\textsc{are } \mathbf{c} }
{\textsc{all } ~\mathbf{a} ~\textsc{are } \mathbf{c} }
\end{displaymath}
\begin{displaymath}
\prftree[r]{$\scriptstyle\uparrow$}
{\textsc{all } ~\mathbf{a} ~\textsc{are } \mathbf{b} }
{~p ~\textsc{is upward monotone}}
{\textsc{if} ~p(\mathbf{a}) ~\textsc{then} ~p(\mathbf{b})}
\end{displaymath}
\begin{displaymath}
\prftree[r]{$\scriptstyle\downarrow$}
{\textsc{all } ~\mathbf{a} ~\textsc{are } \mathbf{b} }
{~p ~\textsc{is downward monotone}}
{\textsc{if} ~p(\mathbf{b}) ~\textsc{then} ~p(\mathbf{a})}
\end{displaymath}
\begin{displaymath}
\prftree[r]{$\scriptstyle Axiom 1$}
{\textsc{all} ~\mathbf{a} ~\textsc{are} ~\mathbf{a}}
\end{displaymath}
\begin{displaymath}
\prftree[r]{$\scriptstyle Axiom 2$}
{\textsc{if} ~p(\mathbf{a}) ~\textsc{then} ~p(\mathbf{a})}
\end{displaymath}

\subsection{Soundness and Completeness}
\begin{mydef}
For a set of $\mathcal{L}$-sentences $\Gamma$, we can define the \emph{canonical model}
$\mathcal{M}_{\Gamma}$ as follows:

First, let $M$ be the set of atomic constant symbols $\mathbf{A}$ and define 
a relation $\leq$ on $\mathbf{A}$ where $\textbf{a} \leq \textbf{b} \iff \Gamma
\vdash a \sqsubset b$. 
The interpretation function $\interp{\cdot}$ is defined as follows:\\

Define $\interp{\textbf{a}} = \downarrow{\textbf{a}} = \{ \textbf{b} \in P \mid \textbf{b} \leq \textbf{a} \}$.

Define $\interp{\sqsubseteq}$ as the $\subseteq$ relation on $\mathcal{P}(M)$.

For each $p \in \mathbf{P}$, we have a conditional definition:

If and only if ``$p$ is upward monotone'' is the only sentence about $p$ entailed by $\Gamma$,
we define $\interp{\sqsubseteq_p} = \subseteq$. 
 
If and only if ``$p$ is downward monotone'' is the only sentence about $p$ entailed by $\Gamma$,
 we define $\interp{\sqsubseteq_p} = \subseteq$. 
 
In all other cases, 
$\interp{\sqsubseteq_p}$ is defined as set equality in $\mathcal{P}(M)$. 
\end{mydef}

\begin{lemm}\label{the_lemma}
For a set $\Gamma$ of $\mathcal{L}$-sentences, the canonical model
$\mathcal{M}_{\Gamma} \models \Gamma$. 
\end{lemm}
\begin{proof}
The key parts of the proof are a consequence of the fact that
$\downarrow a \subseteq \downarrow b \iff a \leq b$, 
and $\downarrow a \supseteq \downarrow b \iff b \leq a$ which is crucial to the case that
$\Gamma$ contains both ``$p$ is upward monotone" and ``$p$ is downward monotone".
The rest is a routine consequence of the definitions.
\end{proof}
\begin{thm}{Soundness and Completeness}
\end{thm}
\begin{proof}
We leave the perfunctory soundness proof as an exercise to 
the reader.
Towards showing completeness, let $\Gamma$ be a set of 
$\mathcal{L}$-sentences and $\phi$ another $\mathcal{L}$-sentence.
Suppose that for every model $\mathcal{M}$
we have that $\Gamma \models \phi$. 
In particular, by lemma \ref{the_lemma}, $\mathcal{M}_\Gamma \models \phi.$ All 
further discussion occurs in this specific model.
The sentence $\phi$ may have one of four forms. 

Suppose firstly that $\phi$
is ``if $p(\mathbf{a})$ then $p(\mathbf{b})$". Thus, $(\interp{\mathbf{a}}, \interp{\mathbf{b}}) 
\in \interp{\sqsubseteq_p}$. Since the interpretation of $\sqsubseteq_p$
depends on the description of $p$ entailed by $\Gamma$, there are three cases:
Firstly, if $\Gamma \vdash$ ``$p$ is upward monotone'' only, then it follows that $\interp{\mathbf{a}}\subseteq \interp{\mathbf{b}}$. Since this holds 
if and only if $a \leq b$ by a basic property of down-sets, 
then we will have $\Gamma \vdash a \sqsubseteq b$ and 
$\Gamma \vdash$ ``$p$ is upward monotone", so that $\Gamma \vdash$ 
``if $p(\mathbf{a})$ then $p(\mathbf{b})$" by the $\uparrow$ deduction
rule.

On the other hand, if we had that $\Gamma \vdash$ ``$p$ is downward 
monotone'' only, we can similarly deduce that $\interp{\mathbf{a}} \supseteq \interp{\mathbf{b}}$, and repeating the same reasoning arrive at 
$\Gamma \vdash$ 
``if $p(\mathbf{a})$ then $p(\mathbf{b})$".
In the last option for $p$, we either have that $\Gamma$ proves neither or
both of the statements ``$p$ is upward monotone" and ``$p$ is downward 
monotone", and in either case $\interp{\sqsubseteq_p}$ is set equality in $\mathcal{M}_\Gamma.$ 
As such, we will be able to conclude 
that $\interp{\mathbf{a}} = \interp{\mathbf{b}}$. Equal
down-sets imply that $a=b$, so that trivially $\Gamma \vdash$ ``if $p(\mathbf{a})$ then $p(\mathbf{b})$"
Hence, in all of these cases, $\Gamma \vdash \phi$. 

If $\phi$ is the sentence ``$p$ is upward monotone'' (we omit the dual, which is similar), 
then truth in the canonical model gives us that $\subseteq = \interp{\sqsubseteq_p}$. In the 
$\mathcal{M}_\Gamma$, this happens exactly when $\Gamma \vdash$ ``$p$ is upward monotone'' .
The last option for $\phi$ is covered in the completeness theorem for the basic syllogistic logic with the
``\textsc{Barbara}" deduction rule and Axiom 1.

In conclusion, in all cases we may deduce that $\Gamma \vdash \phi$. 
\end{proof}

\section{Related Work}
The study of monotonicity in natural language has a strongly developed linguistic and mathematical theoretical groundwork,
dating back to the monotonicity calculus of \cite{valencia} and in semantic studies such as \cite{vanbenthem}.
Its formal treatments have led to the expansion of typed lambda calculus with an order relation so as to model
this order-theoretic behaviour, resulting in a corresponding deduction system and completeness theorem
in \cite{icard-etal-2017-monotonicity}. 
There are varying presentations and some variation in terminology, but for the most part \emph{monotonicity} refers to the
order-theoretic property of the context as a function,
while the term \emph{polarity} usually refers to the tag assigned to the node in a CCG parse tree or a word in a sentence. 
The inferential mechanism based on monotonicity properties of quantifiers, determiners and contexts in general is sometimes
referred to as \emph{natural logic}, and the underlying principles of natural logic applying to set-theoretic 
concept relations has led to work on \emph{generalized monotonicity} \cite{maccartney-manning-2009-extended}. 
However, the additional relations such as negation, alternation and cover are no longer 
order-theoretic notions.

\paragraph{Symbolic Implementations}
There are two flavours of implementations that result in the deductions allowed by monotonicity reasoning.
Firstly, works such as \cite{monalog, abzianidze-2015-tableau} rely on linguistically-informed polarity 
markings on the nodes of CCG parse trees. They require accurate parses and expertly hand-crafted linguistic rules to 
mark the nodes with polarity tags, as in \cite{hu-moss}.
In \cite{monalog}, a premise is tagged for monotonicity and a knowledge base of hypotheses created by a substitution
known to be truth-preserving is generated. Candidate hypotheses are compared with this set, checking for exact 
matches.
On the other hand, \cite{abzianidze-2015-tableau} uses the CCG parses to further translate to a lambda logical form
for use in a deduction method inspired by tableau calculus.
These approaches differ from strategies such as \cite{maccartney-manning-2009-extended}, which require an 
\emph{edit sequence} which transforms the premise into the hypothesis. Atomic edits are
tagged with generalized entailment relations which are combined with a join
operator based on relational composition to determine whether the transformation is overall truth-preserving, 
hence yielding a hypothesis entailed by the premise.
Later, \cite{angeli2014naturalli} treated these atomic edits as edges in a graph and phrased entailment detection
as a graph search problem.
Concepts from symbolic approaches to NLI have also been applied in 
symbolic question answering systems (such as in \cite{bobrow}), and hybridized with 
neural systems (such as in \cite{hynli}).

\paragraph{Neural NLI Models and Monotonicity}
State of the art NLI models have previously been shown \cite{yanakaHELP, geiger} to perform poorly on examples where the 
context $f$ is \emph{downward monotone}, as occurs in the presence
of negation and various generalized quantifiers such as ``every" and ``neither".
Benchmark datasets such as MNLI are somewhat starved of such examples, as observed by 
\cite{yanakaHELP}. As a consequence, the models trained on such benchmark datasets as
MNLI not only fail in downward monotone contexts, but \emph{systematically} fail: 
they tend to treat all examples as if the contexts are upward monotone, predicting the 
\emph{opposite} entailment label with high accuracy \cite{yanakaHELP, geiger}. 
Data augmentation techniques and additional fine-tuning with an inoculation
\cite{inoculation} strategy have been attempted
in \cite{yanakaHELP, larryFragments} and \cite{geiger}.
In the latter case, performance on a challenge test set improved without much performance loss on the 
original benchmark evaluation set (SNLI), 
but in \cite{yanakaHELP} there was a significant decrease in performance on the MNLI evaluation set.
These studies form the basis on which we aim to build, and their choice of evaluation datasets and models
inspires our own choices.

\begin{table}[h]
\centering  
\resizebox{\columnwidth}{!}{%
\begin{tabular}{@{}llllll@{}}
\toprule
                                                                       & \multicolumn{1}{c}{}                                             & \multicolumn{4}{c}{\textbf{Previous Work}}                                                                                                                                                                                                                    \\ \midrule
\multicolumn{2}{c}{\textbf{Evaluation Datasets}}                                                                                          & \begin{tabular}[c]{@{}l@{}}Geiger 2020\\ (Neural)\end{tabular} & \begin{tabular}[c]{@{}l@{}}Yanaka 2020\\ (Neural)\end{tabular} & \begin{tabular}[c]{@{}l@{}}Moss 2019\\ (Neural)\end{tabular} & \begin{tabular}[c]{@{}l@{}}Hu 2020\\ (Symbolic)\end{tabular} \\ \midrule
\begin{tabular}[c]{@{}l@{}}Large, \\ Broad Coverage\end{tabular}       & MNLI Test                                                        &                                                                & x                                                              &                                                              &                                                              \\
                                                                       & \begin{tabular}[c]{@{}l@{}}MNLI Dev\\ (Mismatched)\end{tabular}  &                                                                &                                                                & x                                                            &                                                              \\
                                                                       & SNLI Test                                                        & x                                                              &                                                                & x                                                            &                                                              \\ \midrule
\begin{tabular}[c]{@{}l@{}}Small,\\ Targeted \\ Phenomena\end{tabular} & MED                                                              &                                                                & x                                                              &                                                              &                                                              \\
                                                                       & SICK                                                             &                                                                & x$^*$                                                               &                                                              & x                                                            \\
                                                                       & FraCaS                                                           &                                                                & x$^*$                                                   &                                                              & x                                                            \\
                                                                       & MoNLI Test                                                       & x                                                              &                                                                &                                                              &                                                              \\
                                                                       & \begin{tabular}[c]{@{}l@{}}Monotonicity\\ Fragments\end{tabular} &                                                                &                                                                & x                                                            & x                                                            \\ \bottomrule
\end{tabular}%
}
\caption{Evaluation datasets used in previous work investigating monotonicity reasoning. Positions marked $*$
indicate that the dataset is included in another used evaluation dataset.}
\label{dataset_table}

\end{table}

Neural Transformer-based language models have been shown to implicitly model syntactic structure 
\cite{hewitt}.
There is also evidence to suggest that these NLI models are at least representing
the concept relations quite well and using this information to predict the entailment label, 
as corroborated by a study based on \emph{interchange interventions} in \cite{geiger}.

We hypothesise that such models have the capacity for learning monotonicity features. 
The extent to which the representations capture monotonicity information in the contextual 
representations of tokens in the sequence is not yet well understood, and this is an investigation
we wish to initiate and encourage with this work.

\section{Experiments}
Building on the observations in the above-mentioned previous papers, we ask the following questions:
\begin{itemize}
    \item Can a context monotonicity classification task in the model training pipeline further 
    improve performance on
    targeted evaluation sets which test monotonicity reasoning?
    \item Does this mitigate the decrease in performance on benchmark NLI datasets?
\end{itemize}

Our investigation proceeds in three parts:
Firstly, we attempt to fine-tune a SOTA NLI model for a context monotonicity classification task.

Secondly, we retrain the above model for NLI and evaluate the performance on several evaluation datasets which specifically target
examples of both upward and downward monotonicity reasoning. We examine whether there is any improvement over a previously
suggested approach on fine-tuning on a
large, automatically generated dataset (HELP) from \cite{yanakaHELP}. 

\paragraph{Models}
We start with existing NLI models pretrained on benchmark NLI datasets.
In particular (and for best comparison with related studies) we use RoBERTa \cite{roberta} pretrained 
on MNLI \cite{mnli} and BERT \cite{devlin-etal-2019-bert} pretrained on SNLI \cite{snli}.
These are two benchmark NLI datasets which contain examples derived from naturally occurring text and 
crowd-sourced labels, aiming for scale and broad coverage. 
We do not deviate from the architecture, as we are only investigating the effect of training on different
tasks (monotonicity classification and NLI) and datasets.

\subsection{Retraining NLI Models to Classify Context Monotonicity}\label{classif_section}
Traditionally, symbolic approaches treat monotonicity classification as the task of labeling 
of each node in a CCG parse tree with either an upward or downward polarity marking. 
Our emphasis of monotonicity as a property of a \emph{context} allows for a different framing of 
this problem:
we consider monotonicity classification as a binary classification task by explicitly indicating (with 
a variable) the ``slot" in the sentence for which we wish to know the polarity. 
Different positions of the variable in a partial sentence may yield a context with a different
monotonicity label;
a typical example of this is sentences featuring generalized quantifiers such as ``every", 
which may be monotone up in one argument but monotone down in another. 

\subsubsection{Input Representation}

The NLI models which we wish to start with are transformer-based models, in line with the current state of the art approaches to NLI.
Transformer models represent a sentence as a sequence of tokens: we take a naive approach to representing a context 
by indicating the variable with an uninformative `x' token. 
We refrain from using the mask token to indicate the variable, as the underlying pretrained transformer 
language models are trained to embed the 
mask token in such a way as to correspond with high-likelihood insertions in that position, which we would prefer to avoid.

\subsubsection{Dataset}
In order to ensure our monotonicity classification task does not add any unseen data (when compared to only
fine-tuning on the HELP dataset) we adapt the HELP dataset for this task.
The HELP dataset \cite{yanakaHELP} consists of a set of 
$(p(\mathbf{a}), p(\mathbf{b}))$ pairs which included labels for the entailment relationship between the full sentences, the inclusion relation between $\mathbf{a}$
and $\mathbf{b}$, and the monotonicity classification of $p$.
As such, we extract only the contexts $p$ and the monotonicity label into dataset which we will call ``HELP-contexts'',
which we split into a train, development and test set in a 50:20:30 ratio. Examples of this dataset are presented on Table~\ref{help_contexts_example_table}.
\footnote{The original HELP dataset also contains a few non-monotone examples: 
in the current state of this work, these 
are omitted in favor of a focus on the specific confusion in existing models where downwards 
monotone contexts are often treated as upwards monotone ones.}
\begin{table}[h]
\resizebox{\columnwidth}{!}{%
\begin{tabular}{@{}lll@{}}
\toprule
\textbf{Context}                                 & \textbf{Context Monotonicity} &  \\ \midrule
There were no x today.                  & downward monotone    &  \\
There is no time for x.                 & downward monotone    &  \\
Every x laughed.                        & downward monotone    &  \\
There is little if any hope for his x . & downward monotone    &  \\
Some x are allergic to wheat.           & upward monotone      &  \\
Tom is buying some flowers for x.       & upward monotone      &  \\
You can see some wild rabbits in the x. & upward monotone      &  \\ \bottomrule
\end{tabular}%
}
\caption{Examples from the HELP-contexts dataset, with respective labels.}
\label{help_contexts_example_table}
\end{table}

\subsubsection{Results}

As presented in Table~\ref{context_prediction_table}, the task of predicting the monotonicity 
of a context can be solved using fine-tuned transformer models. 
This suggests a potential path for inducing a bias for context classification
in downstream tasks such as NLI, which could benefit from better modeling of 
context monotonicity.

\begin{table}[h!]
\resizebox{\columnwidth}{!}{%
\begin{tabular}{@{}lcccccc@{}}
\toprule
\multicolumn{1}{c}{\multirow{1}{*}{\textbf{Model}}} & \multicolumn{6}{c}{\textbf{Evaluation Data}}                                                                                                                                         \\
                                                 & \multicolumn{3}{c}{\textbf{\begin{tabular}[c]{@{}c@{}}HELP-Contexts\\ Dev\end{tabular}}} & \multicolumn{3}{c}{\textbf{\begin{tabular}[c]{@{}c@{}}HELP-Contexts\\ Test\end{tabular}}} \\ 
                                                 & Precision & Recall & F1-Score & Precision & Recall & F1-Score \\ \midrule
bert-base                                           &  98.74                        & 99.08                       & 98.91                       & 98.00                        & 95.24                        & 96.54                       \\
bert-large                                          &  98.23                        & 98.88                       & 98.55                       & 97.51                        & 95.70                        & 96.57                       \\
roberta-large-mnli                                  & 99.62                        & 98.92                       & 99.26                       & 98.73                        & 96.64                        & 97.64                       \\
roberta-large                                       &  99.81                        & 99.46                       & 99.27                       & 98.99                        & 96.41                        & 97.62                       \\
roberta-base                                        &  99.81                        & 99.46                       & 99.63                       & 98.10                        & 95.56                        & 96.76                       \\
bert-base-uncased-snli                              &  98.88    & 98.19   & 8.53   & 98.92    & 97.29    & 98.07   \\ \bottomrule
\end{tabular}%
}
\caption{Performance of state of the art models for the context prediction task. Each model was trained on HELP context (training set).}
\label{context_prediction_table}
\end{table}

\subsection{Improving NLI Performance on Monotonicity Reasoning}

A few datasets have been curated to either fine-tune or evaluate NLI models with monotonicity reasoning in
mind: their uses in previous related works are detailed in table \ref{dataset_table}.
We use the following datasets for training and evaluation respectively:

\subsubsection{Training Data}
We start by once again using the HELP dataset \cite{yanakaHELP}, which was designed specifically
as a balanced additional training set for the improvement of NLI models with respect to monotonicity
reasoning. 
We create a split of this dataset which is based on the HELP-contexts dataset by assigning each example
either to the train, development or test set depending on which split its associated context $f$ 
is in the HELP-contexts dataset.
This is to ensure there is no overlap between the examples' contexts accross the three data 
partitions.
Our approach combined this strategy with an additional step based on the context
monotonicity task
described in section \ref{classif_section}. 

\subsubsection{Training Procedure}
We rely on the architecture implementations and pretrained models available with the \emph{transformers} library \cite{transformers}.
Starting with the pretrained models (which we shall henceforth tag as ``bert-base-uncased-snli'' and 
``roberta-large-mnli''),
we first fine-tune these models for the context monotonicity classification task using the training
partition of the HELP-contexts dataset.
We re-use the classification head of the pretrained models for this purpose, but only use two
output states for the classification.

\subsubsection{Evaluation Data}
Evaluation datasets are typically small, challenging and categorized by
certain target semantic phenomena. Following previous work in this area, we evaluate our approach using
the MED dataset introduced in \cite{yanakaMED}, which is annotated with monotonicity information and draws from various
expertly-curated diagnostic challenge sets in NLI such as SICK, FraCaS and the SuperGlue Diagnostic set. 
It features a balanced split between upward and downward
monotone contexts, in contrast to the benchmark MNLI dataset.
Additionally, we include evaluation on the MoNLI dataset \cite{geiger} which also features a labeled balance of upward
and downward monotone examples. However, the latter dataset's
downward monotone examples are only exemplary of contexts
featuring the negation operator \emph{``not"}, whereas 
MED \cite{yanakaMED} also includes more complex downward monotone 
operators such as generalized quantifiers and determiners.
We refer to these respective papers \cite{yanakaMED, geiger} for full breakdowns and analyses of these datasets.

\subsubsection{Baselines}
Although the main comparison to be made is the improvement introduced when including the 
context-monotonicity-classification training on top of the current state-of-the art 
roberta-large-mnli model trained on HELP, we include an additional baselines:
roberta-large-mnli fine-tuned on the \emph{monotonicity 
fragment} from the \emph{semantic fragments} \cite{larryFragments} dataset. 
The strategy in this work is the same as with the HELP dataset, but we include this in the evaluation
on the chosen challenge sets for a more complete comparison.

\subsubsection{Results}
We present the results on the challenge sets MED and MoNLI in Table \ref{nli_table}, with
a break-down by upward and downward monotone contexts.
Furthermore, we have re-run each model on the original benchmark evaluation datasets SNLI and MNLI,
with the results visible in Table \ref{tab:benchmarks}. 

\begin{table*}[ht!]
\resizebox{\textwidth}{!}{%
\begin{tabular}{@{}clcccccc@{}}
\toprule
\multirow{2}{*}{\textbf{Model}} & \multicolumn{1}{c}{\multirow{2}{*}{\textbf{\begin{tabular}[c]{@{}c@{}}Additional \\ Training Data\end{tabular}}}} & \multicolumn{6}{c}{\textbf{Challenge Datasets}}                                                                                                                                                                                     \\

                                & \multicolumn{1}{c}{}                                                                                              & \multicolumn{3}{c}{\textbf{MoNLI Test}}                                                                        & \multicolumn{3}{c}{\textbf{MED}}                                                                               \\
\multicolumn{1}{l}{}            &                                                                                                                    & \multicolumn{1}{l}{Upward Monotone} & \multicolumn{1}{l}{Downward Monotone} & \multicolumn{1}{l}{\textbf{All}} & \multicolumn{1}{l}{Upward Monotone} & \multicolumn{1}{l}{Downward Monotone} & \multicolumn{1}{l}{\textbf{All}} \\ \midrule
\hline
bert-base-uncased-snli          & -                                                                                                                  & 37.74                               & 56.49                                 & 46.15                            & 53.58                               & 43.91                                 & 49.36                            \\
bert-base-uncased-snli          & HELP                                                                                                               & 30.89                               & 85.02                                 & \textbf{55.19}                   & 43.4                                & 72.43                                 & 60.18                            \\
\textbf{bert-base-uncased-snli} & \textbf{HELP + HELP-Contexts}                                                                                      & 21.6                                & 97.67                                 & \textbf{55.19}                   & 32.56                               & 87.13                                 & \textbf{66.22}                   \\ \midrule
roberta-large-mnli              & -                                                                                                                  & 95.19                               & 5.32                                  & 58.84                            & 82.12                               & 25.76                                 & 46.09                            \\
roberta-large-mnli              & Monotonicity Fragments (Easy)                                                                                                                  & 92.68                               & 79.62                                  & 86.81                            & 74.54                               & 65.68                                 & 70.05                               \\
roberta-large-mnli              & Monotonicity Fragments (All)                                                                                                                  & 50.00                               & 50.00                                  & 50.00                            & 35.42                               & 61.80                                 & 49.78                            \\
roberta-large-mnli              & HELP                                                                                                               & 94.72                               & 98.67                                 & 96.48                            & 64.47                               & 86.25                                 & \textbf{77.4}                    \\
\textbf{roberta-large-mnli}     & \textbf{HELP + HELP-Contexts}                                                                                      & 98.78                               & 97.17                                 & \textbf{98.06}                   & 65.24                               & 85.12                                 & 76.44                           \\
\bottomrule
\end{tabular}
}
\caption{Performance of NLI models on challenge datasets designed to test performance on
monotonicity reasoning.}
\label{nli_table}
\end{table*}

\begin{table*}[ht!]
\centering
\resizebox{0.8\textwidth}{!}{%
\begin{tabular}{@{}lccccccccc@{}}
\toprule
\multicolumn{1}{c}{}                                 &                                                                                                & \multicolumn{8}{c}{\textbf{Benchmark Datasets}}                                                                                                                                                   \\
\multicolumn{1}{c}{\multirow{-2}{*}{\textbf{Model}}} & \multirow{-2}{*}{\textbf{\begin{tabular}[c]{@{}c@{}}Additional \\ Training Data\end{tabular}}} & \multicolumn{2}{c}{\textbf{MNLI (m$^*$) Dev}}      & \multicolumn{2}{c}{\textbf{MNLI (mm$^*$) Dev}}     & \multicolumn{2}{c}{\textbf{SNLI Dev}}          & \multicolumn{2}{c}{\textbf{SNLI Test}}         \\
                                                     & \multicolumn{1}{l}{}                                                                          & Acc                          & $\Delta$            & Acc                          & $\Delta$             & Acc                          & $\Delta$            & Acc                          & $\Delta$             \\ \midrule
{\color[HTML]{000000} bert-base-uncased-snli}        & {\color[HTML]{000000} -}                                                                       & {\color[HTML]{000000} 44.96} & -               & {\color[HTML]{000000} 45.52} & -               & {\color[HTML]{000000} 41.54} & -               & {\color[HTML]{000000} 40.78} & -               \\
bert-base-uncased-snli                               & HELP                                                                                           & 35.13                        & -9.83           & 34.37                        & -11.5           & 25.93                        & -15.61          & 25.92                        & -14.86          \\
\textbf{bert-base-uncased-snli}                      & \textbf{HELP  + HELP-Contexts}                                                                 & 36.91                        & \textbf{-8.05}  & 37.36                        & \textbf{-8.16}  & 36.54                        & \textbf{-5.00}  & 37.20                        & \textbf{-3.58}  \\ \midrule
roberta-large-mnli                                   & -                                                                                              & 94.11                        & -               & 93.88                        & -               & 93.33                        & -               & 93.14                        & -               \\
roberta-large-mnli                                   & HELP                                                                                           & 82.66                        & \textbf{-11.45} & 83.38                        & \textbf{-10.50} & 74.77                        & -18.56          & 74.39                        & -18.75          \\
\textbf{roberta-large-mnli}                          & \textbf{HELP  + HELP-Contexts}                                                                 & 81.00                        & -13.11          & 82.01                        & -11.87          & 82.99                        & \textbf{-10.34} & 82.31                        & \textbf{-10.83} \\
\bottomrule
\end{tabular}%
}
\caption{Fine-tuning state of the art NLI models with the aim of improving monotonicity has tended
to result in 
lower performance on the original benchmark NLI datasets. We compare these performance losses in addition to 
tracking performance on the the challenge datasets.
$^*$ MNLI (m) and (mm) refers to the matched and mismatched dataset respectively. For MNLI, only the \emph{Dev} set is publically available.}
\label{tab:benchmarks}
\end{table*}


\section{Discussion}

\paragraph{Average Performance}
Firstly, we confirm previous observations that the starting pretrained transformer
model roberta-large-mnli (which is considered
a high-performing NLI model, achieving over $93\%$ accuracy on the large MNLI development set)
has a dramatic performance imbalance with respect to context monotonicity.
The fact that performance on downward monotone contexts is as low as $5\%$ suggests that this model
perhaps routinely assumes upward monotone contexts.
It was noted in \cite{yanakaHELP} that the MNLI benchmark dataset is strongly skewed in favor of
upward monotone examples, which may account for this. 

Our approach outperforms or matches the baseline models in three of the summary accuracy scores,
and is competitive in the fourth. 
Furthermore, in most cases we observe less performance loss on the benchmark sets.

\paragraph{Performance by Monotonicity Category}

As evident from Table~\ref{nli_table}, we observe a substantial improvement for the bert-base-uncased
NLI models for downward monotone contexts.  
For the much larger roberta-large-mnli models, any gains over the model trained on HELP only are
quite small.
A common observation is the notable trade-off between accuracy on upward and downward monotone contexts;
training that improves one of these over a previous baseline generally seem to decrease performance of the 
other.
This is especially evident in the MED dataset, which is larger and representative of a more diverse
set of downward monotone examples (the MoNLI dataset is limited to the ``No" operator). 
Sensibly, a decrease in performance in upward monotone contexts also leads to a decrease in performance
on the original SNLI and MNLI datasets  \ref{tab:benchmarks} (which are skewed in favor of upward monotone examples). 
However, in most cases (except for the roberta-large-mnli model on the MNLI benchmark) our method
results in a \emph{smaller} performance loss. 

\section{Conclusion and Future Work}
Introducing context monotonicity classification into the training pipeline of NLI models 
provides some performance
gains on challenge datasets designed to test monotonicity reasoning.
We see contexts as crucial objects of study in future approaches to natural language inference.
The ability to detect their logical properties (such as monotonicity)
opens the door for hybrid neuro-symbolic NLI models and reasoning systems, 
especially in so far as dealing with out of domain insertions that may confuse 
out-of-the-box NLI models. The linguistic flexibility that transformer-based language models bring is too 
good to lose; leveraging their power in situations where only part of our sentence is in 
a model's distribution would be helpful for domain-specific use cases with many out-of-distribution nouns.
Overall, we are interested in furthering both the \emph{analysis} and \emph{improvement} of 
emergent modelling of abstract logical features in neural natural language processing models.

\bibliographystyle{acl_natbib}
\bibliography{acl2021}

\end{document}